\documentclass[conference]{IEEEtran}
\IEEEoverridecommandlockouts
\usepackage{cite}
\usepackage{amsmath,amssymb,amsfonts}
\usepackage{graphicx}
\usepackage{amsthm}
\newtheorem{assumption}{Assumption}

\newtheorem{theorem}{Theorem}

\usepackage{textcomp}
\usepackage{lipsum}
\usepackage{algorithm}
\usepackage{algpseudocode}
\usepackage{subfig}
\usepackage{xcolor}
\def\BibTeX{{\rm B\kern-.05em{\sc i\kern-.025em b}\kern-.08em
    T\kern-.1667em\lower.7ex\hbox{E}\kern-.125emX}}

\begin{document}

\title{ Online Vertical Federated Learning for Cooperative Spectrum Sensing\\
}

\author{\IEEEauthorblockN{Heqiang Wang\IEEEauthorrefmark{1},
Jie Xu\IEEEauthorrefmark{1}}
\IEEEauthorblockA{\IEEEauthorrefmark{1}Department of Electrical and Computer Engineering,\\
University of Miami, Coral Gables, FL 33146, USA}}

\maketitle

\begin{abstract}
The increasing demand for wireless communication underscores the need to optimize radio frequency spectrum utilization. An effective strategy for leveraging underutilized licensed frequency bands is \textit{cooperative spectrum sensing} (CSS), which enable multiple \textit{secondary users} (SUs) to collaboratively detect the spectrum usage of \textit{primary users} (PUs) prior to accessing the licensed spectrum. The increasing popularity of machine learning has led to a shift from traditional CSS methods to those based on deep learning. However, deep learning-based CSS methods often rely on centralized learning, posing challenges like communication overhead and data privacy risks. Recent research suggests \textit{vertical federated learning} (VFL) as a potential solution, with its core concept centered on partitioning the deep neural network into distinct segments, with each segment is trained separately. However, existing VFL-based CSS works do not fully address the practical challenges arising from streaming data and the objective shift. In this work, we introduce \textit{online vertical federated learning} (OVFL), a robust framework designed to address the challenges of ongoing data stream and shifting learning goals. Our theoretical analysis reveals that OVFL achieves a sublinear regret bound, thereby evidencing its efficiency. Empirical results from our experiments show that OVFL outperforms benchmarks in CSS tasks. We also explore the impact of various parameters on the learning performance. 
\end{abstract}


\section{Introduction}
As the demand for wireless communication continues to grow, optimizing the utilization of the radio frequency spectrum has become imperative. One effective solution for efficiently utilizing underutilized licensed frequency bands is \textit{cooperative spectrum sensing} (CSS)~\cite{akyildiz2011cooperative}, wherein \textit{secondary users} (SUs) work together to detect spectrum usage by \textit{primary users} (PUs) before accessing the licensed spectrum. In CSS, a \textit{fusion center} (FC) collects spectrum sensing data, such as \textit{received signal strength} (RSS) and location information, from multiple SUs, and then uses this aggregated data to make decisions regarding spectrum allocation and usage. With recent advancements in \textit{machine learning} (ML), CSS fusion strategies are increasingly incorporating ML techniques, where the fusion center trains a \textit{deep neural network} (DNN) to make spectrum-related decisions based on the SUs' sensing data \cite{lee2019deep}. However, these centralized training approaches necessitate SUs to transmit their raw sensing and location data to the FC, posing challenges like increased communication loads and potential privacy risks.

\begin{figure}[htbp]
\vspace{-5pt}
\centering
\subfloat{\includegraphics[width=0.98\linewidth]{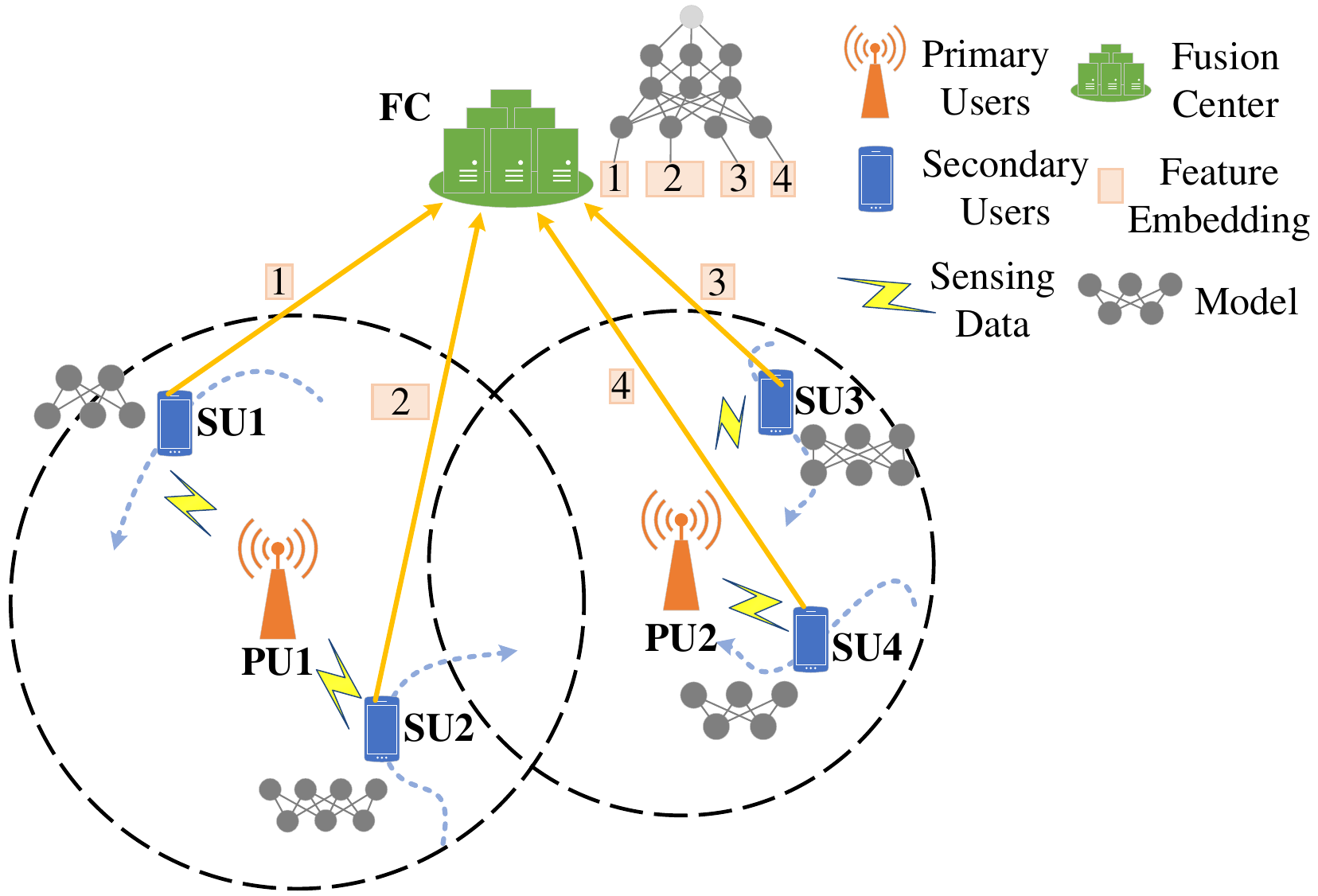}} 
\caption{The VFL-based mobile CSS scenario.}  
\label{ovfl-vf}
\vspace{-15pt}
\end{figure}

To address the aforementioned challenges, \textit{vertical federated learning} (VFL) has been recently proposed to tackle the CSS problem~\cite{zhang2020vertical, zhang2022low}. The core concept of VFL involves partitioning the DNN into distinct segments, each of which is trained separately by both individual SUs and the FC, thereby preserving the raw data at the SUs. More specifically, the DNN is partitioned into a head classifier and multiple vertically separable feature extractors, one for each SU, as illustrated in Fig.~\ref{ovfl-vf}. Instead of transmitting raw sensing data to the FC, each SU can now feed its local sensing results to its own feature extractor, subsequently forwarding the extracted feature, referred to as a feature embedding, to the FC. The feature embeddings from multiple SUs are then processed by the head classifier at the FC to derive the final spectrum decision. Given that feature embeddings are typically much smaller in size compared to raw data, this approach eases the communication load. Furthermore, the preservation of raw data at the SU significantly enhances data privacy.

It is important to highlight that VFL represents a substantial departure from conventional \textit{federated learning} (FL), or horizontal FL, and it is particularly well-suited for tackling the CSS problem. In traditional horizontal FL, all clients share the same feature space and DNN model architecture. However, for CSS, the sensing results obtained by different SUs exhibit varying patterns and can significantly differ based on the SUs' respective locations. Integrating the sensing results from all SUs leads to a more accurate determination of the PUs' states compared to relying solely on the sensing result from a single SU. Therefore, considering the sensing results of all SUs at a specific point in time as a single data sample, each SU's sensing result effectively represents a portion of the features within that data sample.

While VFL holds significant promise in addressing the CSS problem, existing approaches have overlooked several critical practical constraints. Firstly, SUs are typically edge devices with limited storage capacity, continuously receiving and recording data samples on-the-fly instead of having access to a complete training dataset from the outset. Therefore, VFL must operate with dynamic datasets that are constructed from an ongoing data stream, rather than relying on static datasets. Secondly, and perhaps even more critically, the underlying data distributions in the CSS learning problem can undergo changes due to the mobility of SUs and the evolving wireless environment, including factors such as shadowing effects between the SUs and PUs. These variations inherently lead to shifts in learning objectives. Consequently, traditional VFL solutions for CSS that rely on offline training methods become ineffective, emphasizing the need for an online learning-based VFL algorithm to address these challenges.

In this work, we propose \textit{online vertical federated learning} (OVFL) as a robust framework tailored to address the CSS problem while accommodating the aforementioned constraints. Our primary contributions can be summarized as follows:
(1) We formulate the CSS learning problem as an OVFL problem and present an OVFL algorithm for online training. This algorithm is specifically engineered to process streaming data and adapt to shifts in learning objectives. To further optimize communication efficiency for feature embedding transmission, we incorporate a quantization method. (2) We provide a comprehensive theoretical analysis of the OVFL algorithm, considering both scenarios with and without quantization. We demonstrate that, by selecting a learning rate $\eta = \mathcal{O}(1/\sqrt{T})$, OVFL without quantization achieves a sublinear regret bound over $T$ time rounds, specifically $\mathcal{O}(\sqrt{T})$. Additionally, in the scenario of OVFL with quantization, achieving a sublinear regret bound is feasible, provided the quantization level decreases appropriately. (3) We present empirical results that highlight OVFL's superiority over benchmark methods in addressing CSS problems. Additionally, we conduct an in-depth exploration of the impact of various parameters and scenarios on the performance of the OVFL algorithm.

The rest of this paper is organized as follows. In Section II, we discuss related works on CSS, VFL and \textit{online federated learning} (OFL). Section III presents the system model and formulates the OVFL problem. In Section IV, we introduce the detail of OVFL algorithm. Section V presents the regret analysis of OVFL in scenarios both with and without quantization. The experimental results of OVFL are presented in Section VI. Finally, Section VII concludes the paper.

\section{Related Work}
\subsection{Cooperative Spectrum Sensing}
CSS has become a critical technique for the efficient utilization of underutilized licensed frequency bands, improving the reliability of PU signal detection by leveraging the spatial diversity among SUs \cite{akyildiz2011cooperative}. The CSS problem has garnered significant attention within the wireless communication research community \cite{shi2020machine, zhang2022speckriging}. Traditional CSS approaches, including energy detection \cite{atapattu2011energy}, compressed sensing \cite{tian2007compressed}, and matched filter detection \cite{kapoor2011opportunistic}, are well-regarded for their simplicity and efficacy in diverse environments. However, the emergence of deep learning has changed CSS, introducing advanced solutions to overcome the constraints of traditional methods. Deep learning architectures, especially CNNs and RNNs, excel in discerning the temporal and spatial patterns in spectrum data, thereby enhancing detection accuracy and robustness \cite{sarikhani2020cooperative, lee2019deep, shi2020machine, zhang2022speckriging}. Concurrently, the expansion of distributed learning scenarios, such as FL, has prompted researchers to utilize VFL to address CSS problem \cite{zhang2022low, zhang2020vertical}. The VFL-based CSS approach mitigates the communication overhead and privacy risks inherent in data transfer within centralized deep learning frameworks. Regarding prior research, the author in \cite{zhang2020vertical} introduced a VFL-based CSS scheme that surpasses conventional methods, albeit without providing any theoretical analysis. Meanwhile, the author in \cite{zhang2022low} concentrated primarily on a truncated VFL algorithm, which significantly decreases training latency in VFL-based CSS schemes by omitting parties whose channel gains fall below a certain threshold.

\subsection{Vertical Federated Learning}
In recent years, VFL has attracted significant attention. Introduced by \cite{hardy2017private}, VFL operates on vertically partitioned data, a concept distinct from horizontal FL. Subsequent comprehensive surveys \cite{yang2023survey, wei2022vertical, liu2022vertical} have further explored VFL's scope. Unlike horizontal FL, VFL presents unique challenges. Some studies \cite{feng2022vertical, kang2022fedcvt, fan2022fair} have sought to optimize data utilization to improve the efficacy of the joint model in VFL. Others \cite{sun2022label, sun2021defending, lu2020multi} have concentrated on developing privacy-preserving protocols to mitigate data leakage risks. Efforts have also been made to reduce communication overhead, either through multiple local updates per iteration \cite{castiglia2022flexible, zhang2022adaptive} or via data compression techniques \cite{liu2020accelerating, castiglia2022compressed}. Furthermore, due to its practical benefits in enabling data collaboration among diverse institutions across multiple industries, VFL has attracted heightened interest from both academic and industrial communities, finding applications in a range of fields. These include, but are not limited to, recommendation systems \cite{zhang2021vertical, cui2021exploiting}, finance \cite{ou2020homomorphic}, and healthcare \cite{chen2020vafl, hu2022vertical}. However, existing VFL research mainly rely on static datasets, thus overlooking practical constraints like streaming data \cite{wang2023local} and objective shifts \cite{ruan2021towards}. Consequently, developing an online learning-based VFL solution that tackles these challenges is imperative.

\subsection{Online Federated Learning}
Online learning is designed to process data sequentially and update models incrementally, making it particularly well-suited for applications where data arrives continuously, and models must adapt to new patterns on-the-fly \cite{sahoo2017online}. These methods are computationally efficient and have the advantage of not requiring the entire dataset to be available at the start of the learning process, which is ideal for scenarios with limited memory resources. In the domain of FL, \textit{online federated learning} (OFL) \cite{hong2021communication} has emerged as a novel paradigm that extends the principles of online learning to a network of multiple learners or agents. The primary distinction between OFL algorithms and the traditional FL algorithms (such as FedAvg \cite{li2019convergence}, SCAFFOLD \cite{karimireddy2020scaffold}, FedProx \cite{li2020federated}) lies in the objective of local updates. Whereas the traditional FL algorithms focus on finding a single global model that minimizes a global loss function, OFL algorithms strive to identify a sequence of global models that minimize the cumulative regret. Limited work in recent years studied the problem of OFL. For example, the authors of \cite{kwon2023tighter} present a communication-efficient OFL method that balances reduced communication overhead with robust performance. In another instance, \cite{mitra2021online} introduces FedOMD, an OFL algorithm for uncertain environments that processes streaming data from clients without making statistical assumptions about loss functions. To the best of our knowledge, existing OFL approaches are primarily relevant to horizontal FL scenarios and cannot be directly applied to VFL scenarios.

\section{System Model}
We consider a wireless network consisting of a FC, $N$ PUs, and $K$ SUs. Time is discretized into periods, denoted as $t = 1, 2, ..., T$. Owing to the mobility of devices, the positions of PUs and SUs may change between time periods, resulting in a dynamic wireless environment. This variability can impact the evolving learning objectives, which will be discussed shortly.

In each time period, each SU $k$ collects a local training dataset consisting of $M$ sensing data samples \cite{zhang2022low, lee2019deep}, represented as $\textbf{x}^t_k \in \mathbb{R}^{M \times Q_k}$, where $Q_k$ is the dimension of the raw sensing data. The individual sensing data samples, denoted as $x^{t, m}_k$ for all SUs, are collected simultaneously and paired with a common label, $y^{t, m}$, representing the PUs' activities, such as their transmission power levels. The collective training dataset is denoted as $\textbf{x}^t \in \mathbb{R}^{M \times Q}$, where $Q = \sum_{k=1}^K Q_k$. It is important to note that the collective training dataset $\textbf{x}^t$ is introduced for conceptual clarity, as the local training datasets are retained by the respective SUs and are not shared with the FC during the training process.

In the VFL framework, each SU trains a distinct feature extractor model characterized by the parameter $\theta_k$ for its raw sensing data. Meanwhile, the FC trains a head classifier model represented by the parameter $\theta_0$. The overall model's parameters are collectively denoted as $\Theta = [\theta^\top_0, \theta^\top_1, ..., \theta^\top_K]^\top$. Let $h_k(\theta_k; x^{t,m}_k)$ represent the feature embedding extracted from the sample $x^{t,m}_k$. This feature embedding operation transforms the high-dimensional raw data into a lower-dimensional representation through multiple layers of DNN, effectively capturing essential information from the input data while significantly reducing its dimension. Utilizing these feature embeddings, we can express the loss functions for the collective training dataset at period $t$ as follows:
\begin{align}
F_t(\Theta; \textbf{x}^t, \textbf{y}^t)  = \frac{1}{M}\sum_{m=1}^M l_t(\theta_0, \{h_k(\theta_k; x^{t,m}_k)\}_{k=1}^K; y^{t,m})
\end{align}
where $l_t(\cdot)$ is the loss function for a single data sample. 

To streamline notations, we adopt several simplifications in the subsequent sections of this paper. (1) We denote the feature embedding for the dataset $\textbf{x}^{t}_k$ as $h_k(\theta; \textbf{x}^t_k)$ and often use the shorthand $h_k(\theta_k; \textbf{x}^t_k) = h^t_k(\theta_k)$. (2) We assign $k = 0$ to the FC, defining $h_0(\theta_0) = \theta_0$. It is important to note that $h_0(\theta_0)$ represents the head model rather than the feature embedding. (3) We compactly represent $F_t(\Theta; \textbf{x}^t, \textbf{y}^t)$ as $F_t(\Theta)$ to emphasize the dependency of the loss function on the overall model parameter $\Theta$. Additionally, we frequently express $F_t(\Theta) = F_t(h_0(\theta_0), h_1(\theta_1), ..., h_K(\theta_K))$.


Consider the sequence of models trained on the dynamic dataset denoted as $\Theta^1, \ldots, \Theta^T$. Learning regret, denoted by $\text{Reg}_T$, is defined as the discrepancy between the cumulative loss incurred by the learner and the cumulative loss of an optimal fixed model in hindsight. In other words:
\begin{align}
  \text{Reg}_T = \sum_{t=1}^{T}  F_t (\Theta^t; \textbf{x}^t, \textbf{y}^t ) - \sum_{t=1}^{T} F_t (\Theta^*; \textbf{x}^t, \textbf{y}^t) \label{regret}
\end{align}
Here, $\Theta^* = \arg\min_\Theta \sum_{t=1}^{T} F_t (\Theta; \textbf{x}^t, \textbf{y}^t)$ represents the optimal fixed model retrospectively. Our objective is to minimize learning regret, which is equivalent to minimizing the cumulative loss. Notably, if the learning regret exhibits sublinear growth concerning $T$, it suggests that the online learning algorithm can asymptotically minimize the training loss, even when training data is sequentially available over time.

\section{Online Vertical Federated Learning}
In this section, we introduce our OVFL algorithm. To enhance the clarity of the upcoming explanations, we define a time period as equivalent to a global training round in VFL. During each global round, every SU and FC can conduct a specified number of local training iterations denoted by the parameter $E$. We will use the index $\tau = 0, 1, 2, ..., E$ to track these local iterations. The OVFL algorithm is detailed in Algorithm \ref{alg: ovfl}. In each global round $t$, the OVFL algorithm works as follows. The time diagram of steps within a single global round is visually represented in Fig.~\ref{tl_gr}.

\begin{figure}[htbp]
\vspace{-2pt}
\centering
\subfloat{\includegraphics[width=0.98\linewidth]{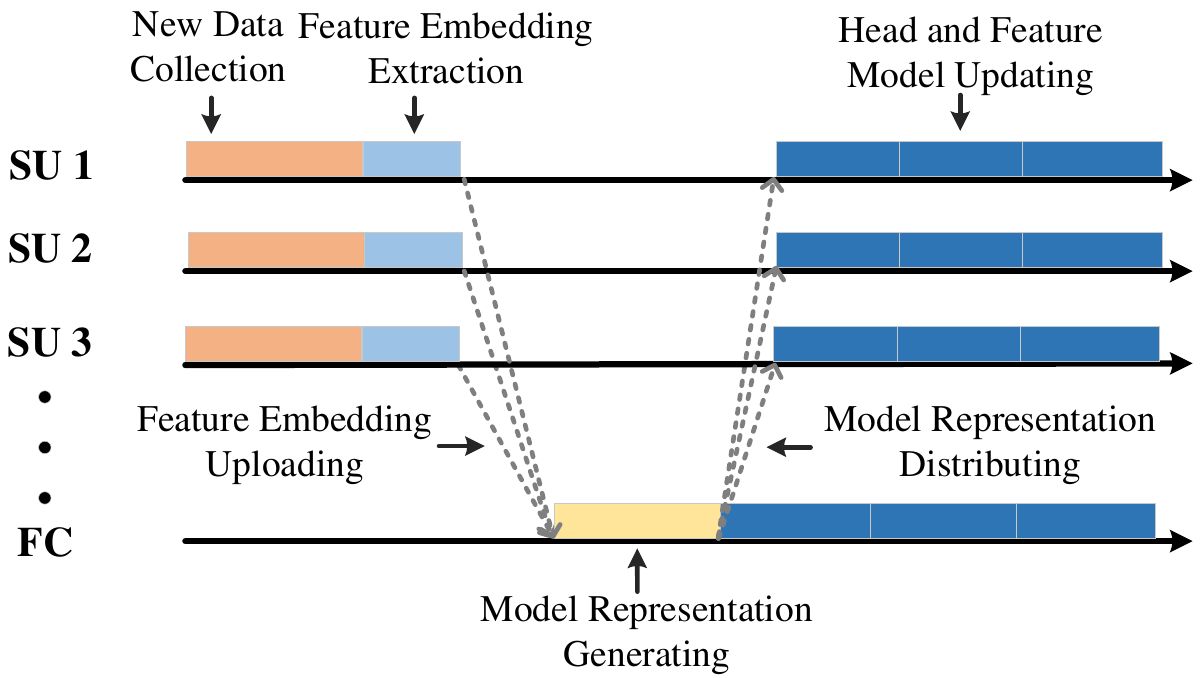}} 
\caption{The time diagram of steps in each global round $t$.}  
\label{tl_gr}
\vspace{-5pt}
\end{figure}

\begin{algorithm}
    \caption{Online Vertical Federated Learning} \label{alg: ovfl}
    \begin{algorithmic}[1]
        \State \textbf{Initialize}: The initial feature model $\theta_k^{t=0,\tau=0}$ for all SUs $k$ and the initial head model $\theta_0^{t=0,\tau=0}$ for FC.
        \For {$t =0, 1, 2, ..., T - 1$}
            \For {$\tau = 0, 1, 2, ..., E-1$}
                \If {$\tau = 0$}
                    \For {$k = 1, 2, ..., K$ in parallel}
                        \State Collects new data samples $\textbf{x}_k^{t}$
                        \State Gets the feature embedding ${h}_k ({\theta}_k^{t, 0}; \textbf{x}_k^t)$
                        \State Quantizes the embedding $\hat{h}_k ({\theta}_k^{t, 0}; \textbf{x}_k^t)$
                        \State Sends $\hat{h}_k ({\theta}_k^{t,0}; \textbf{x}_k^t)$ to FC
                    \EndFor
                    \State FC collects model representation $\hat{\Phi}^{t,0}$
                    \State FC sends $\hat{\Phi}^{t,0}$ to all SUs
                \EndIf
                \For {$k =0, 1, 2, ..., K$ in parallel}
                    \State Gets $\hat{\Phi}^{t, \tau}_k \leftarrow \left \{ \hat{\Phi}^{t,0}_{-k}; {h}_k ({\theta}_k^{t, \tau})\right \}$
                    \State Updates head or feature model ${\theta}_k^{t, \tau+1}$ 
                \EndFor
            \EndFor
        \EndFor
    \end{algorithmic}
\end{algorithm}

\subsubsection{Embedding Computation}
At the beginning of each global round $t$, each SU $k$ gathers a fresh local training dataset $\textbf{x}_k^{t}$ and $\textbf{y}^t$. Subsequently, each SU computes the feature embeddings for these new data samples, represented as $h^t_k(\theta_k^{t, 0})$, employing its current local feature extractor $\theta^{t,0}_k$. The initial feature model $\theta^{t, 0}_k$ for the current global round $t$ is inherited from the end of the previous global round $\theta^{t-1, E}_k$. To optimize communication efficiency, the feature embeddings undergo quantization before being transmitted to the FC. We refer to these quantized feature embeddings as $\hat{h}^t_k(\theta_k^{t, 0})$. Various types of quantizers, including scalar or vector quantizers can be utilized here.

It is important to note that, in contrast to the horizontal FL scenario where quantization is applied to the model or gradient, in VFL, quantization is implemented on the feature embeddings. This distinction results in a completely different set of challenges, which we will discuss in detail in the subsequent analysis of regret bound.

\subsubsection{Model Representation Distributing}
Upon collecting all quantized feature embeddings, the FC quantizes the current head model, designated as $\hat{\theta}^{t, 0}_0 = \hat{\theta}^{t-1, E}_0$. Furthermore, the FC compiles the model representation $\hat{\Phi}^{t, 0}$, which includes the quantized head model and all quantized feature embeddings. More precisely, the model representation is defined as follows:
\begin{align}
\hat{\Phi}^{t,0}\leftarrow \left \{ \hat{\theta}_0^{{t,0}},   \hat{h}_1^t ({\theta}_1^{{t,0}}), \cdots, \hat{h}_k^t ({\theta}_k^{{t,0}}), \cdots,   \hat{h}_K^t ({\theta}_K^{{t,0}})\right \}  
\end{align}
Subsequently, the FC distributes the model representation $\hat{\Phi}^{t, 0}$ to all SUs. The collection of feature embeddings from all SUs and the head model from FC, excluding $k$, is represented as $\hat{\Phi}^{t, 0}_{-k}$.

\subsubsection{Head and Feature Model Updating}
Each SU $k$ or the FC employs the received model representation $\hat{\Phi}^{t, 0}$ to update its own feature or head model for $E$ iterations, following this formula for all $\tau = 0, ..., E-1$:
\begin{align}
    \theta_k^{t, \tau+1} = \theta_k^{t, \tau} - \eta \nabla_k F_t\left (\hat{\Phi}^{t,0}_{-k}, {h}_k^t(\theta_k^{t, \tau} ) \right ) \label{lu}
\end{align}
For convenience, we abuse the notation $F_t$, and define $\nabla_k F_t(\cdot)$ to be the partial derivatives for parameters $\theta_k$.  

It is crucial to acknowledge that OVFL assumes all FC and SUs have access to the labels. In low-risk scenarios, such as the CSS problem, the need for label privacy among the FC and SUs might be negligible. However, in situations where labels are private, OVFL can be enhanced by incorporating the method from \cite{liu2019communication} for gradient computation, which obviates the need for sharing labels.

It is also worth mentioning that both the head model update and the feature model update utilize online gradient descent (OGD) \cite{ying2008online}, which computes the gradient using the current data samples. This is in contrast to the conventional offline approach where stochastic gradient descent is typically used, involving the sampling of a subset from a static dataset for gradient computation. 

\section{Regret Analysis}
In this section, a thorough regret analysis of OVFL is presented. We begin by evaluating the scenario without the impact of quantization, and subsequently analyze the scenario that includes the effects of quantization.

\subsection{OVFL without the impact of quantization}
In this scenario, each SU or FC performs $E$ iterations of local model update in every global round. We introduce some additional definitions and assumptions to support our analysis. We begin by defining the model representation for SU or FC at global round $t$ and local iteration $\tau$ as follows:
\begin{align}
{\Phi}^{t, \tau}_k\leftarrow \left \{ {\theta}_0^{{t, 0}},   {h}_1^t ({\theta}_1^{{t, 0}}), \cdots, {h}_k^t ({\theta}_k^{{t, \tau}}), \cdots,   {h}_K^t ({\theta}_K^{{t, 0}})\right \}  
\end{align}

To facilitate the analysis, we also define ${\textbf{G}}^{t, \tau}$ as the stacked partial derivatives at global round $t$ and local iteration $\tau$:
\begin{align}
   {\textbf{G}}^{t, \tau} := \left [ (\nabla_{0}  F_t ({\Phi}^{t, \tau}_0))^{\top}, \dots, (\nabla_{K}  F_t ({\Phi}^{t, \tau}_K))^{\top} \right ]^{\top}
\end{align}
For the sake of convenience, we denote ${\textbf{G}}^{t, \tau}_k = \nabla_{k}  F_t ({\Phi}^{t, \tau}_k)$. Subsequently, the overall model $\Theta$ updates in every local iteration can be represented as follows:
\begin{align}
\Theta^{t, \tau +1} = \Theta^{t, \tau} - \eta  {\textbf{G}}^{t, \tau}
\end{align}
By summarising the above equation for one global round $\tau = 0, ..., E-1$, we obtain:
\begin{align}
\Theta^{t+ 1, 0} = \Theta^{t, 0} - \eta  \sum_{\tau = 0}^{E -1} {\textbf{G}}^{t, \tau}
\end{align}

For the purposes of subsequent theoretical analysis, we consider a $D$-dimensional vector in both the overall gradient and model. We define an arbitrary vector element $d \in [1, D]$ in overall gradient as ${\textbf{G}}_{k, d}$, and similarly, the arbitrary vector element $d \in [1, D]$ in overall model is denoted as $\Theta_{k, d}$.

Subsequently, we will introduce the assumptions that are standard for analyzing online convex optimization, as referenced in \cite{park2022fedqogd}. Some assumptions are defined at the vector element level, tailored to the requirements of our proof.

\begin{assumption}
For any $(\textbf{x}^t; \textbf{y}^t)$, the loss function $F_t (\Theta; \textbf{x}^t; \textbf{y}^t)$ is convex with respect to $\Theta$ and differentiable.
\label{assm:cov}
\end{assumption}

\begin{assumption}
The loss function is $L$-Lipschitz continuous, the partial derivatives satisfies: $\left \| \nabla_{k}  F_t (\Theta)   \right \|^2 \leq L^2$.
\label{assm:bpd}
\end{assumption}

\begin{assumption}
The partial derivatives, corresponding to the consistent loss function, fulfills the following condition:
\begin{align}
\left \|  {\textbf{G}}_{k}^{t, \tau'} - {\textbf{G}}^{t, \tau}_{k} \right \| \leq \epsilon \left \| \theta_k^{t, \tau'} - \theta_k^{t, \tau} \right \| \notag
\end{align}
\label{assm:gradient-change}
In the context of the online learning scenario, where the loss function evolves over time, we utilize $t$ to indicate that gradients and models correspond to a consistent loss function. To denote that they originate from different local iterations, we employ $\tau'$ and $\tau$ respectively.
\end{assumption}

\begin{assumption}
The arbitrary vector element $d$ in the overall model $\Theta_{k, d}$ is bounded as follows: $\left | \Theta_{k, d} \right | \leq \beta $.
\label{assm:model-variant}
\end{assumption}

Assumption \ref{assm:cov} guarantees the convexity of the function, thereby allowing us to utilize the properties associated with convexity. Assumption \ref{assm:bpd} limits the magnitude of the loss function's partial derivatives. Assumption \ref{assm:gradient-change} ensures that the variation in the partial derivatives is confined within a specific range, which aligns with the model variation over two different local iterations that maintain a consistent loss function. This approach effectively utilizes the concept of smoothness. Lastly, Assumption \ref{assm:model-variant} specifies the permissible range for any vector element in the overall model. Based on the aforementioned assumptions, we can derive the following Theorem 1.

\begin{theorem}\label{thm:general}
Under Assumption 1-4, OVFL with local iterations $E>1$ and excluding the impact of quantization, achieves the following regret bound:
\begin{align}
 & {Reg}_T   = \sum_{t=1}^{T} \mathbb{E}_t \left [ F_t (\Theta^{t,0}; \textbf{x}^t; \textbf{y}^t)  \right ] - \sum_{t=1}^{T} F_t (\Theta^*; \textbf{x}^t; \textbf{y}^t) \notag \\
 & \leq \frac{  \left \| \Theta^{1, 0} - \Theta^* \right \|^2}{2 \eta E }  + \frac{1}{2}\eta T E K L^2 + 2 \eta T \beta D \epsilon E L
\end{align}
\end{theorem}
\begin{proof}
The proof can be found in Appendix \ref{appex-theo1}.
\end{proof}

According to Theorem \ref{thm:general}, by setting $\eta = \mathcal{O}(1/\sqrt{T})$, the OVFL can achieve a sublinear regret bound over $T$ time rounds, specifically $\mathcal{O}(\sqrt{T})$. A sublinear regret bound indicates that the average regret per round, defined as the regret divided by the number of rounds, approaches zero as the number of rounds increases indefinitely. This suggests that the algorithm progressively refines its performance by learning from its errors. 

\subsection{OVFL with the impact of quantization}
Next, we present the proof for the scenario where $E > 1$, taking into account the impact of quantization. To facilitate this proof, we need to further introduce some additional notations. Initially, We begin by defining the model representation for SU or FC at global round $t$ and local iteration $\tau$, incorporating the effects of quantization, as follows:
\begin{align}
\hat{\Phi}^{t, \tau}_k\leftarrow \left \{ \hat{\theta}_0^{{t, 0}},   \hat{h}_1^t ({\theta}_1^{{t, 0}}), \cdots, {h}_k^t ({\theta}_k^{{t, \tau}}), \cdots,   \hat{h}_K^t ({\theta}_K^{{t, 0}})\right \}  
\end{align}
Then we also define $\hat{\textbf{G}}^{t, \tau}$ as the stacked partial derivatives at global round $t$ and local iteration $\tau$ with the impact of quantization:
\begin{align}
   \hat{\textbf{G}}^{t, \tau} := \left [ (\nabla_{0}  F_t (\hat{\Phi}^{t, \tau}_0))^{\top}, \dots, (\nabla_{K}  F_t (\hat{\Phi}^{t, \tau}_K))^{\top} \right ]^{\top}
\end{align}
Then, within the quantization scenario, the update of the overall model $\Theta$ in every local iteration becomes:
\begin{align}
\Theta^{t, \tau +1} = \Theta^{t, \tau} - \eta  \hat{\textbf{G}}^{t, \tau}
\end{align}
By summarizing the above equation for one global round $\tau = 0, ..., E-1$, we can get:
\begin{align}
\Theta^{t+ 1, 0} = \Theta^{t, 0} - \eta  \sum_{\tau = 0}^{E -1} \hat{\textbf{G}}^{t, \tau}
\end{align}
Here we introduce an additional assumption to bound the impact of quantization on the gradient.
\begin{assumption}
The arbitrary vector element $d$ in the overall gradient, both with and without the impact of quantization, demonstrates a finite range of variation: $\left |\hat{\textbf{G}}^{t, \tau}_{k, d} - {\textbf{G}}^{t, \tau}_{k, d}  \right | \leq \rho$.
\label{assm:gradient-quanti}
\end{assumption}

Assumption \ref{assm:gradient-quanti} guarantees that the variation in each vector element of the gradient remains within a pre-specified range when comparing gradients with and without quantization. Then we can further deduce the following Theorem 2, taking into account the impact of quantization.

\begin{theorem}\label{thm:general-quanti}
Under Assumption 1-5, OVFL employing local iteration $E>1$ and taking into account the impact of quantization, achieves the following regret bound:
\begin{align}
 & {Reg}_T   = \sum_{t=1}^{T} \mathbb{E}_t \left [ F_t (\Theta^{t,0}; \textbf{x}^t; \textbf{y}^t)  \right ] - \sum_{t=1}^{T} F_t (\Theta^*; \textbf{x}^t; \textbf{y}^t) \notag \\
 & \leq \frac{  \left \| \Theta^{1, 0} - \Theta^* \right \|^2}{2 \eta E }   + \eta T E K L^2  + \eta T E D \rho^2 \notag \\
& + 2  \eta T \beta D  \epsilon E L + 2 T \beta D \rho
\end{align}
\end{theorem}
\begin{proof}
The proof can be found in Appendix \ref{appex-theo2}.
\end{proof}
According to Theorem 2, by setting $\eta = \mathcal{O}(1/\sqrt{T})$, the OVFL with quantization can achieve the regret bound $\mathcal{O}(\sqrt{T} + T \rho)$ over $T$ time rounds. In this study, we find that due to the presence of quantization, the term $\mathcal{O}( T \rho)$ is critical in determining whether a sublinear regret bound can be achieved. In the workflow of VFL, quantization is applied to feature embedding and the head model instead of being applied directly to gradients. This results in an inability to precisely determine the rate at which $\rho$ in relation to specific changes in the quantization level. However, it is important to note that by gradually reducing the quantization level, a gradient quantization error can be achieved that aligns with the rate of $\rho = \mathcal{O}(1/\sqrt{T})$. Therefore, OVFL with quantization can ultimately achieve a sublinear regret in scenarios that involve appropriately decaying quantization levels.

\section{Simulation Results}
In this section, we evaluate the performance of the OVFL algorithm in the context of CSS involving mobile SUs. We simulate a wireless network spanning an area of $500$m $\times$ $500$m, wherein 4 SUs continuously monitor the states of 2 PUs. Each PU $n$ with fixed locations operates at one of four distinct transmit power levels, denoted as $P_n^{tr} \in \left \{ 1, 2, 3, 4  \right \} $. Then the RSS of SU $k$ is determined using the following path loss model:
\begin{align}
P_{k}^{re} = \sum_{n \in N}\left (  P_{k,n}^{tr} - 10 \times \phi  \times \log_{10}(d_{k, n}) - X_{se} \right )
\end{align}
Where the pathloss exponent $\phi $ is set to  4. The shadowing effect is denoted as $X_{se}$, follows a normal distribution that varies according to location and demonstrates a mean value ranging from 0 to 10 dB, as shown in Fig.~\ref{semv}. Neighboring regions typically display similar distributions of shadowing effects. SUs are characterized by their mobility, moving randomly at mobility rate $v$ each global round. At their respective locations, these SUs collect RSS as local feature vectors. Importantly, within a single global round $t \in T$, the channel states and positions of all SUs remain constant.

\begin{figure}[htbp]
\vspace{-10pt}
\centering
\subfloat{\includegraphics[width=0.7\linewidth]{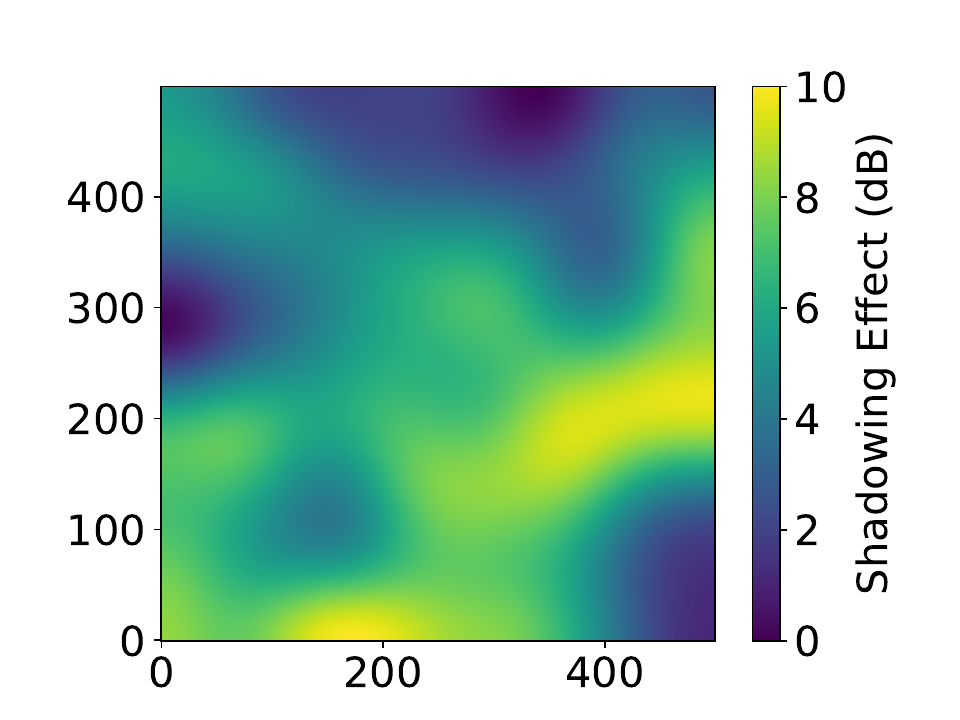}} 
\caption{A map of shadowing effect in space.}  
\label{semv}
\vspace{-5pt}
\end{figure}

Assume SU $k$ relocates to a new position at the start of the global round $t, \tau=0$ with a mobility rate $v$ and begins gathering RSSs at the current location. Fig.~\ref{stp} illustrates the trajectories of the SUs and the positions of the PUs. To accumulate sufficient data at each global round, each SU needs to collect RSSs across multiple slots, each characterized by varying transmit power level of the PUs. Additionally, within each slot, SUs also need to collect RSSs under a consistent transmit power level of the PUs. In each slot, the RSS is recorded as a feature for every SU, capturing energy perception data at various locations concurrently. According to this collection principle, each SU records 40 slots of RSS data samples at its current global round, amidst PUs with varying transmit power levels. Within each slot, each SU collects 100 RSSs as features,  in the presence of PUs maintaining consistent transmit power levels. Therefore, the local feature vector for each SU comprises 102 dimensions, incorporating the SU's own 2D location and the 100 RSSs. The label for each data sample spans 2 dimensions, reflecting the transmit power levels of the two PUs within a single slot. In each global round $t$, SUs capture 20 data samples for training, and 20 data samples for testing.

\begin{figure}[htbp]
\vspace{-10pt}
\centering
\subfloat{\includegraphics[width=0.7\linewidth]{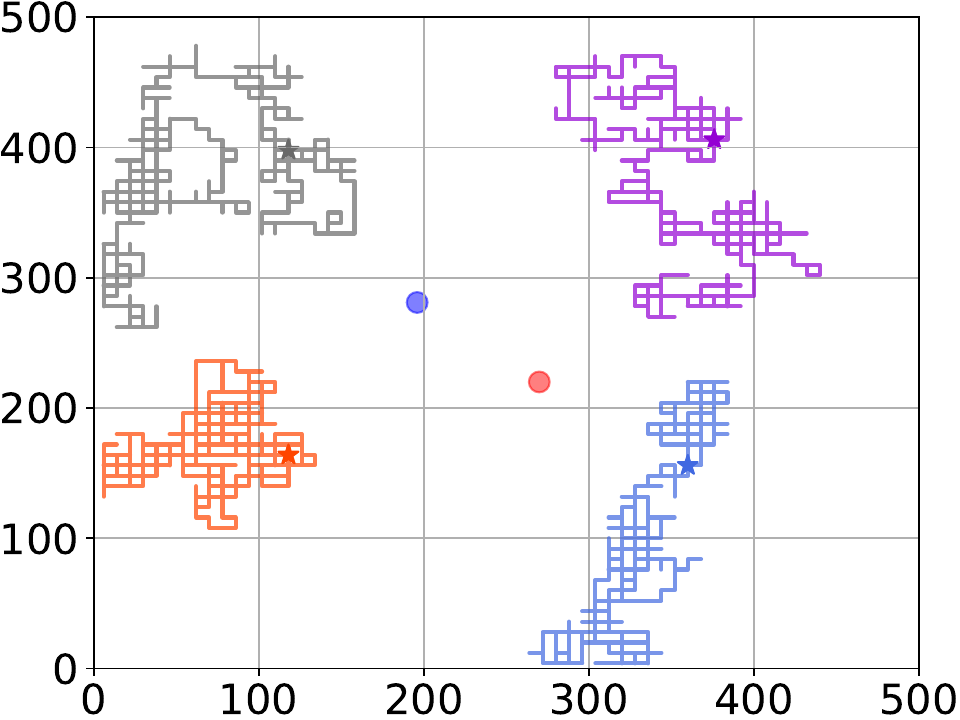}} 
\caption{A map of the SU's trajectory in space.}  
\label{stp}
\vspace{-10pt}
\end{figure}

We now delineate the specifics of the models and quantizers employed in our experiments. For both the SU and the FC, we utilize fully connected networks (FCNs). Each SU is equipped with a 5-layer FCN with size $(102, 128, 256, 64, 16)$. Conversely, the FC incorporates a 3-layer FCN of size $(16 \times 4, 8, 2)$. For quantization, we first employ an uniform scalar quantizer. In our experiments, each embedding component is represented as a 32-bit float. Let $b$ denote the bits per component to which we compress. In the context of the uniform scalar quantizer, this implies $2^b$ quantization levels. When quantization is not applied, $b$ is equal to 32. The training parameters are configured as follows: local iterations $E=[1, 4]$, mobility rate $v = [1, 5]$, global round $T=300$, learning rate $\eta = 0.0001$, and the number of bits per component, $b = [2, 4]$. For the loss function, we consider using mean squared error (MSE) to evaluate learning performance.

\textbf{Benchmarks}. In the experiment, the following benchmarks are used for performance comparison. 
\begin{enumerate}
    \item \textbf{Centralized Cooperative Spectrum Sensing (CC)}. In this scenario, each SU employs conventional training methods to transmit data to the FC, which then integrates the training. While our experiments do not explicitly consider data privacy breaches in this process, it is important to note that the act of data transmission inherently increases the risk of privacy leakage.
    \item \textbf{Lazy Cooperative Spectrum Sensing (LC)}. In this scenario, each SU initially updates the online model. However, after a specified duration, the model discontinues its updates and exclusively depends on the previously trained model for testing.
\end{enumerate}

In our experiments, we initially explore the effects of the aforementioned parameter variations. Furthermore, we delve into the impacts of additional parameters on learning performance to more comprehensively demonstrate the characteristics of our proposed OVFL algorithm. 

\textbf{Performance comparison}. We first evaluate the regression test loss of the OVFL with benchmarks with $E=1$ and $v=1$. Given the aforementioned simulation setup, the performance comparison between OVFL and benchmarks is shown in Fig.~\ref{ovfl-pcr} and Fig.~\ref{ovfl-pcc}. Specifically, Fig.~\ref{ovfl-pcr} illustrates the relationship between test loss and the number of rounds, whereas Fig.~\ref{ovfl-pcc} depicts the correlation between test loss and communication cost. We derive several key observations from these figures. 

\begin{figure}[htbp]
\vspace{-10pt}
\centering
\subfloat{\includegraphics[width=0.7\linewidth]{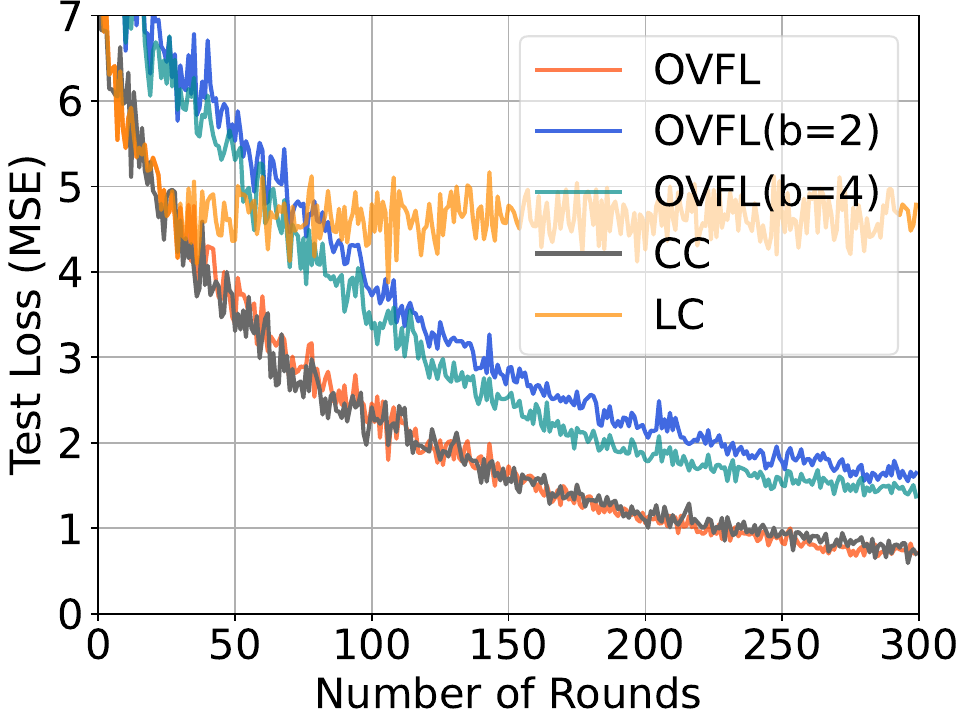}} 
\caption{Performance comparison between OVFL and benchmarks by rounds.}  
\label{ovfl-pcr}
\vspace{-5pt}
\end{figure}

From Fig.~\ref{ovfl-pcr}, we observe that, given the online nature of the problem, the LC approach utilizing the previous model incurs a substantial test loss, markedly inferior to other methods. Furthermore, OVFL and CC exhibit comparable behaviors. In comparisons of OVFL with quantization, the test loss rate declines as $b$ decreases, primarily because a smaller $b$ results in a larger quantization error. Therefore the adoption of the quantization technique will always have an impact on learning performance.

\begin{figure}[htbp]
\vspace{-10pt}
\centering
\subfloat{\includegraphics[width=0.7\linewidth]{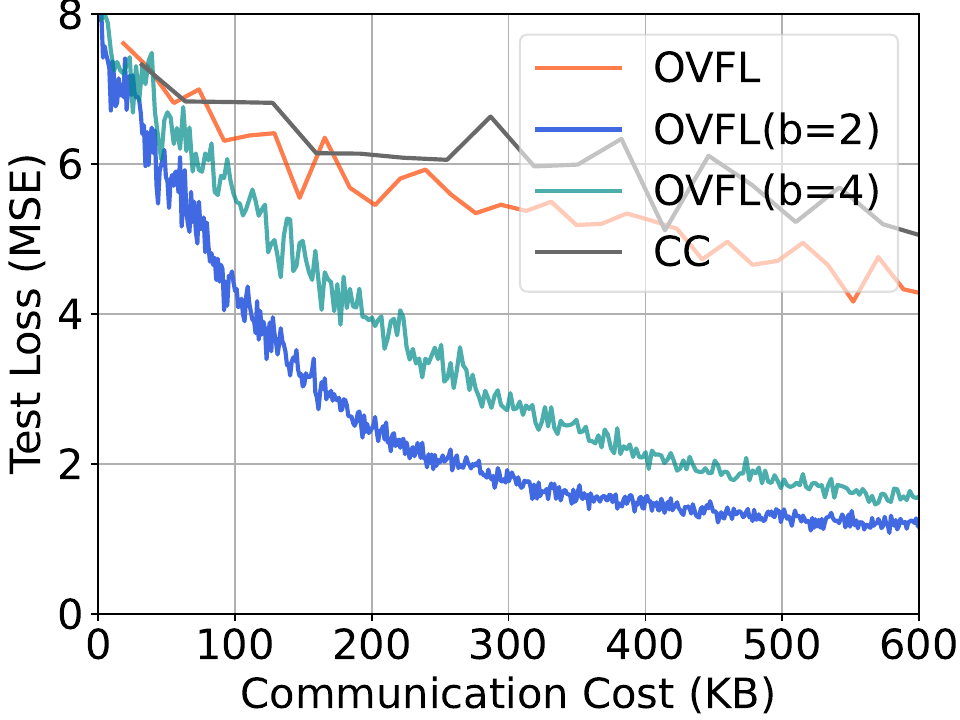}} 
\caption{Performance comparison between OVFL and benchmarks by communication cost.}  
\label{ovfl-pcc}
\vspace{-5pt}
\end{figure}

Referring to Fig.~\ref{ovfl-pcc}, it becomes apparent that at a constant communication cost, $b=2$ offers the best test loss due to its minimal communication overhead. This advantage is ascribed to the negligible disparity in test loss between $b=2$ and $b=4$, along with the fact that the communication cost for $b=2$ is merely half compared to that for $b=4$. By contrast, OVFL without quantization underperforms in this case than OVFL with quantization, which can be attributed to the absence of any quantization technique. CC method falls behind all other approaches. Given that the communication cost associated with data transfer exceeds that of feature embedding transfer, CC incurs a higher communication overhead and, additionally, presents an elevated risk of privacy leakage. 

Consequently, it is evident that a trade-off exists between learning performance and communication cost. When prioritizing learning performance, it is advisable to avoid excessively low levels of quantization. Conversely, if communication constraints are a significant bottleneck for training, it may be necessary to compromise on learning performance and employ quantization methods.

Based on the experimental results, it is evident that our proposed algorithm outperforms the benchmarks. Subsequent sections will present further comparisons of OVFL without quantization under various settings.

\begin{figure}[htbp]
\vspace{-5pt}
\centering
\subfloat{\includegraphics[width=0.7\linewidth]{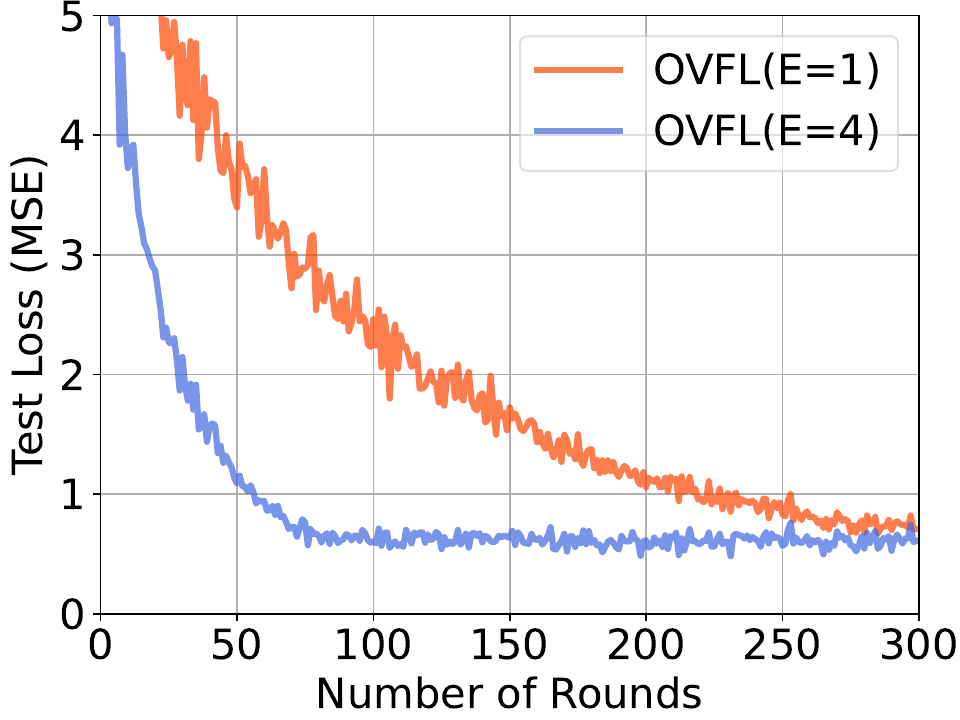}} 
\caption{Performance comparison of OVFL with different $E$.}  
\label{ovfl-e}
\vspace{-5pt}
\end{figure}

\textbf{Impact of local epoch $E$}. In this section, we evaluate the impact of a local epoch setting of $E \in \left [1, 4 \right ]$ on the learning performance while keeping the $v=1$. Fig.~\ref{ovfl-e} illustrates the relationship between test loss and round $T$, indicating that a higher $E$ value facilitates faster convergence without increasing communication costs. This is achieved by maximizing local training in each round $t$ within the given training duration, thereby accelerating convergence without adding to the communication overhead.

\begin{figure}[htp]
\vspace{-5pt}
\centering
\subfloat{\includegraphics[width=0.7\linewidth]{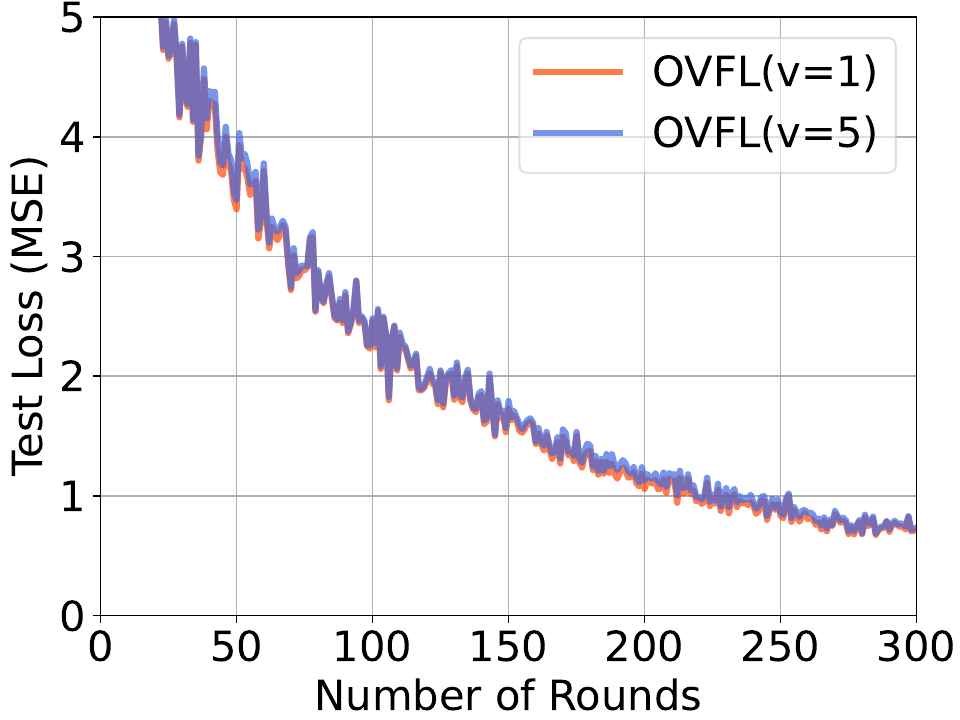}} 
\caption{Performance comparison of OVFL with different $v$.}  
\label{ovfl-v}
\vspace{-5pt}
\end{figure}

\textbf{Impact of mobility rate $v$}. In this set of experiments, we investigate the impact of mobility rate $v \in \left [1, 5 \right ]$ on the learning performance while keeping the $E=1$. The results, illustrated in Fig.~\ref{ovfl-v}, demonstrate a similarity in test loss, indicating that our online algorithm consistently adapts well across a range of $v$ values, further underscoring the superiority of our proposed OVFL algorithm.

\textbf{Impact of type of quantizer}. In this section, we employ a vector quantizer called 2-dimensional hexagonal lattice quantizer \cite{shlezinger2020uveqfed}. This quantizer operates on a hexagonal lattice structure, assigning the pre-quantization value to the nearest internal lattice coordinates based on Euclidean distance. Similar, $b$ represents the bits per component used for compression. In the context of the 2-dimensional hexagonal lattice quantizer, this means there are $2^{2b}$ two-dimensional vectors. Here we investigate the impact of $b \in \left [2, 4 \right ]$ on the learning performance. The results, as illustrated in Fig.~\ref{ovfl-vqr} and Fig.~\ref{ovfl-vqr-com}, reveal that within the vector quantizer setting, performance is akin to that of a scalar quantizer, where a large $b$ results in better performance, but higher communication cost.

We note that learning performance with vector quantizers is significantly better than with scalar quantizers. This superior performance of vector quantizers is likely due to their improved capability to efficiently handle multidimensional data, adapt to diverse data distributions, and minimize quantization errors. Another potential reason for the enhanced performance of vector quantizers over scalar quantizers is that scalar quantizers represent the original raw data using only $b$ bits, whereas vector quantizers require $2b$ bits.

\begin{figure}[htp]
\vspace{-5pt}
\centering
\subfloat{\includegraphics[width=0.7\linewidth]{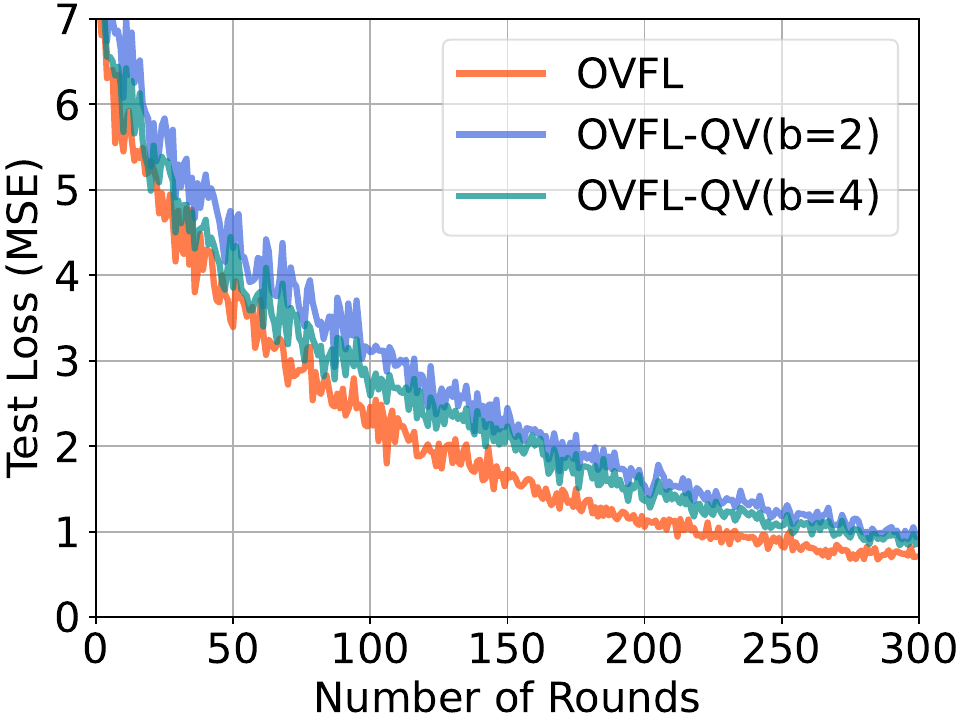}} 
\caption{Performance comparison of OVFL in vector quantizer setting with different $b$ by round.}  
\label{ovfl-vqr}
\vspace{-2pt}
\end{figure}

\begin{figure}[htp]
\vspace{-5pt}
\centering
\subfloat{\includegraphics[width=0.7\linewidth]{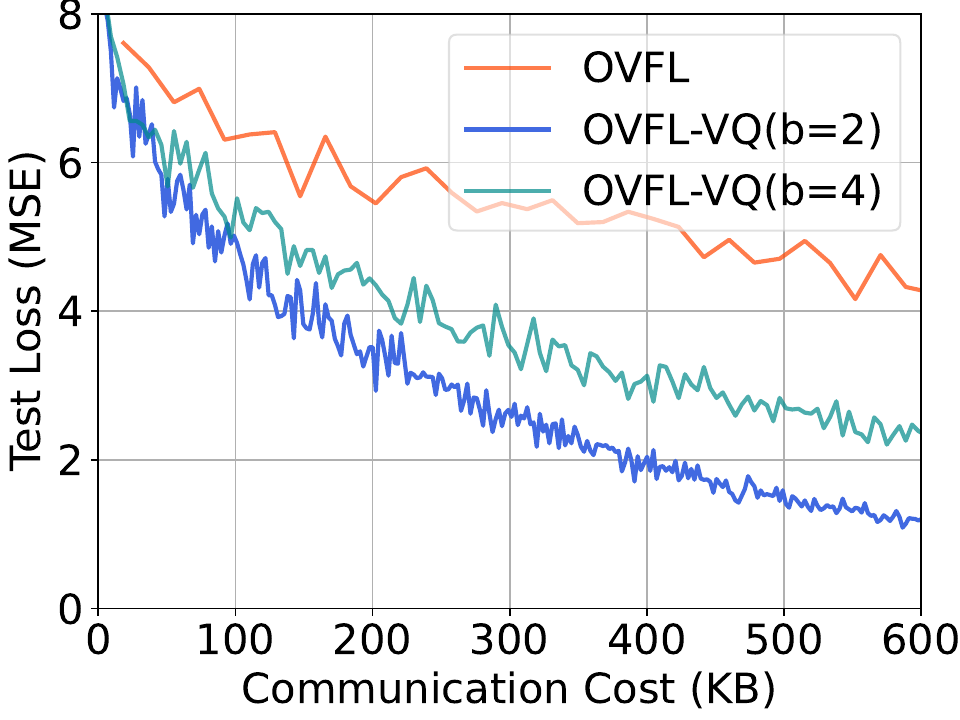}} 
\caption{Performance comparison of OVFL in vector quantizer setting with different $b$ by communication cost.}  
\label{ovfl-vqr-com}
\vspace{-2pt}
\end{figure}

\textbf{Impact of number of PUs}.
In this section, we explore the impact of varying the number of PUs on learning performance, while maintaining a constant number of SUs at 4. Specifically, we investigate scenarios with the number of PUs set to either 1, 2 or 4, while keeping all other parameters constant. The results are presented in Fig.~\ref{ovfl-pu}. The learning performance deteriorates with the increase in the number of PUs, which can be attributed to the fact that when there is only one PU, the SU can accurately assess the PU's transmit power levels through the RSS. However, as the number of PUs increases, the superposition of RSS signals complicates the relationship between the aggregated RSS and the transmit power levels of each individual PU. Consequently, this complexity renders the learning process more challenging and diminishes learning performance.

\begin{figure}[htp]
\vspace{-5pt}
\centering
\subfloat{\includegraphics[width=0.7\linewidth]{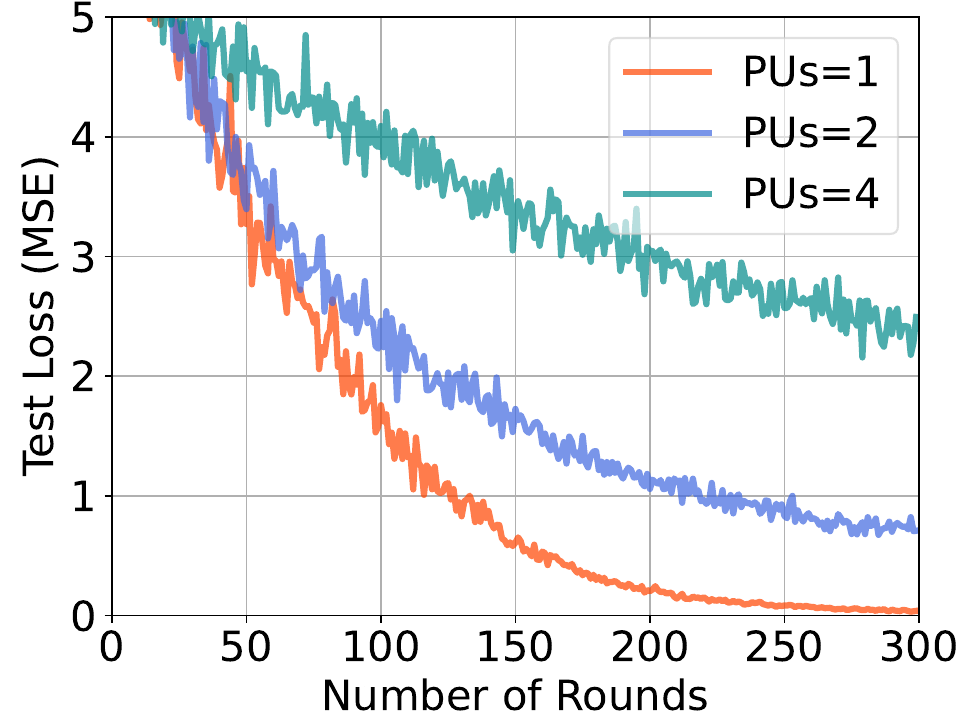}} 
\caption{Performance comparison of OVFL with different number of PUs.}  
\label{ovfl-pu}
\vspace{-5pt}
\end{figure}

\textbf{Impact of number of SUs}.
In this section, we explore the impact of varying the number of SUs on learning performance, while maintaining a constant number of PUs at 2. Specifically, we investigate scenarios with the number of SUs set to either 2, 4 or 6, while keeping all other parameters constant. The results, as displayed in Fig.~\ref{ovfl-su}, reveal that learning performance enhances with an increase in the number of SUs when compare the results between $2$ and $4$. However, beyond a certain threshold in the number of SUs, further enhancements in learning performance become limited despite additional increases in SUs. These experiments lead us to conclude that, generally, a higher number of PUs requires a corresponding increase in the number of SUs, but beyond a specific threshold, the increase in the number of SUs contributes only marginally to learning performance improvements.

\begin{figure}[htp]
\vspace{-5pt}
\centering
\subfloat{\includegraphics[width=0.7\linewidth]{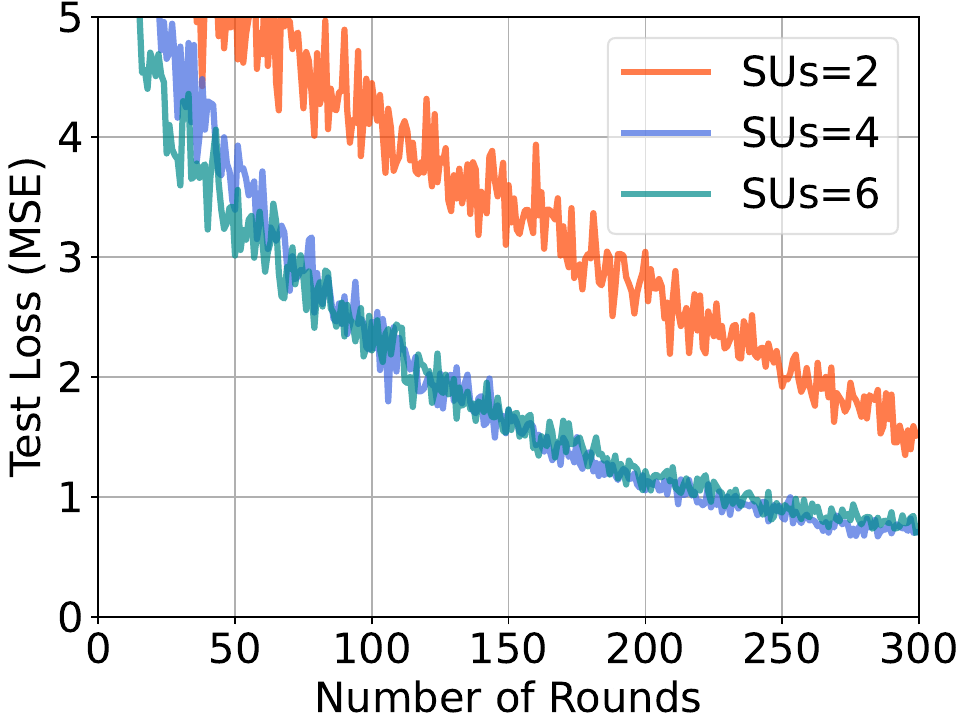}} 
\caption{Performance comparison of OVFL with different number of SUs.}  
\label{ovfl-su}
\vspace{-5pt}
\end{figure}

\textbf{Study in smart publication transportation}. 
In this study, we build upon previous work \cite{bian2022mobility} that examined smart public transportation scenarios. Here, the SUs in our algorithm are modeled based on real-world movement trajectories of four actual bus routes operating in Miami between West Kendall (WK) and Dadeland South (DS). We assume that each SU represents a corresponding bus on each bus line, with each bus traveling from WK to DS at its respective speed. It is posited that we have collected data across 100 global rounds and conducted 100 global rounds of training throughout the process. The two PUs are located near WK and DS, respectively. The bus lines are illustrated in Fig.~\ref{bus_trace}. In this instance, we examine the OVFL scenario with $E = 4$ and no quantization. The learning performance is depicted in Fig.~\ref{bus_lp}. The results demonstrate that the test loss decreases effectively, consistent with the outcomes of our prior experiments.

\begin{figure}[htp]
\vspace{-5pt}
\centering
\subfloat{\includegraphics[width=0.9\linewidth]{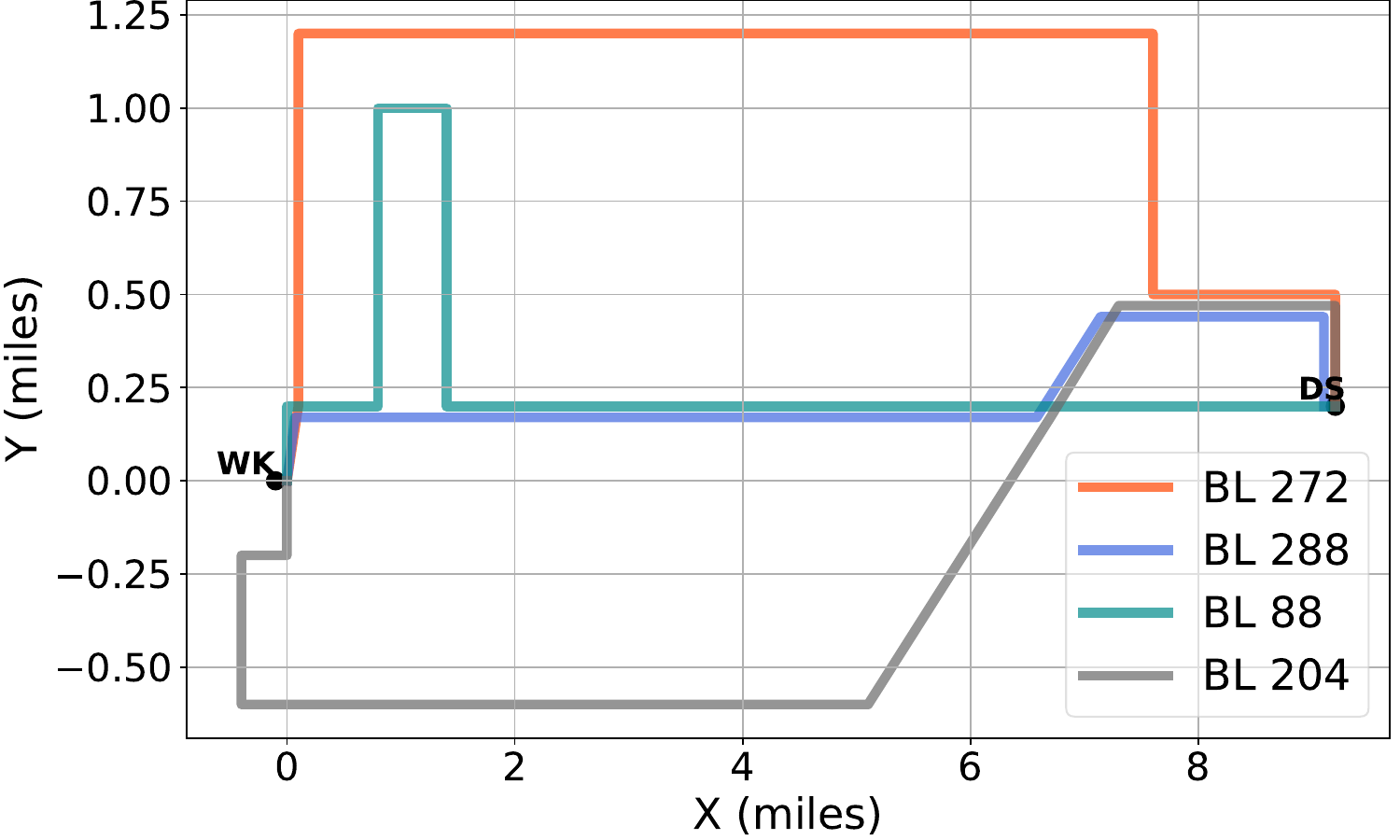}} 
\caption{An illustration of smart public transportation with four real bus lines in Miami.}  
\label{bus_trace}
\vspace{-5pt}
\end{figure}

\begin{figure}[htp]
\vspace{-5pt}
\centering
\subfloat{\includegraphics[width=0.7\linewidth]{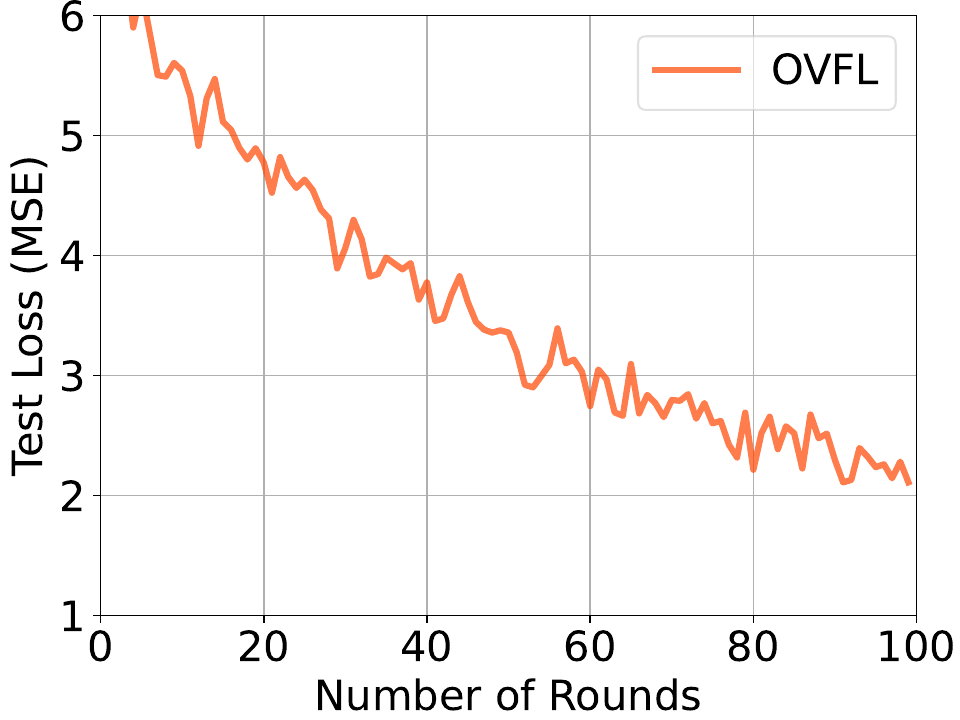}} 
\caption{Learning performance of OVFL under the smart publication transportation.}  
\label{bus_lp}
\vspace{-5pt}
\end{figure}

The aforementioned experimental results systematically and comprehensively illustrate the superiority of the OVFL algorithm, as well as the impact of various parameters and scenarios on the experimental outcomes.

\section{Conclusion}
Our paper presents a novel VFL algorithm named \textit{online vertical federated learning} (OVFL), which differs from traditional VFL by operating on the dynamic dataset and shift learning objective. We provide a detailed description of the workflow of OVFL and present an in-depth analysis of its regret, both with and without the application of quantization. The findings indicate that a sublinear regret bound is attainable in scenarios without quantization. In scenarios with quantization, selecting an appropriate decay rate for the quantization level enables the achievement of a sublinear regret bound in the final result. In our experiments, we employed OVFL to address the cooperative spectrum sensing problem. We also investigate the algorithm's performance across a range of parameters and scenarios, all of which consistently affirmed the superiority of our proposed OVFL algorithm. 

Future research on OVFL will follow several directions: firstly, it will delve into the implementation of more advanced online learning algorithms within VFL scenarios, aiming for a faster rate of regret reduction. Secondly, the research will focus on expanding OVFL to a wider range of application scenarios within wireless communication and addressing the unique challenges inherent to each scenario. Lastly, while the current approach has addressed communication efficiency, further research is needed to achieve computational efficiency in OVFL through model pruning and knowledge distillation, customized for the specific workflow of OVFL.

\bibliographystyle{IEEEtran}
\bibliography{bibligraphy}


\setcounter{page}{1}
\newpage
\appendices

\section{Proof of Lemma 1}
In the Lemma 1, our objective is to establish the validity of the following inequality.
\begin{align}
( {\textbf{G}}^{t, \tau})^{\top} (\Theta^{t, 0} - \Theta^*) \geq   ({\textbf{G}}^{t, 0})^{\top} (\Theta^{t, 0} - \Theta^*) - 2 \eta \beta D \epsilon \notag
\end{align}
Next, we detail the specific steps of the proof. Initially, drawing from Assumption \ref{assm:model-variant}, we derive the following:
\begin{align}
\left | \Theta^{t, 0}_{k, d} - \Theta^*_{k, d} \right | \leq 2 \beta
\end{align}
By definition of ${\textbf{G}}^{t, \tau}_{k, d}$ and Assumption \ref{assm:gradient-change}, we get the following inequality:
\begin{align}
& \left |{\textbf{G}}^{t, \tau}_{k, d}  - {\textbf{G}}^{t, 0}_{k, d} \right | \leq \left \|  {\textbf{G}}_{k}^{t, \tau} - {\textbf{G}}^{t, 0}_{k} \right \| \\
& \leq \epsilon \left \| \theta_k^{t, \tau} - \theta_k^{t, 0} \right \| \\
& \leq \eta \epsilon \left \| \sum_{\tau = 0}^{E-1} \nabla_{k}  F_t ({\Phi}^{t, \tau}_k)  \right \| \\
& \leq \eta \epsilon E L
\end{align}
Where the last inequality is due to the Assumption 2. By integrating the aforementioned inequality, we can deduce the following:
\begin{align}
\left |({\textbf{G}}^{t, \tau}_{k, d}  - {\textbf{G}}^{t, 0}_{k, d} ) (\Theta^{t, 0}_{k, d} - \Theta^*_{k, d} )\right | \leq 2 \eta \beta \epsilon E L \label{combine}
\end{align}
Where Eq.~(\ref{combine}) is hold due to the fact $\left | A B \right | \leq \left | A \right | \left | B \right |$. Then we can get the lower bound of ${\textbf{G}}^{t, \tau}_{k, d}  (\Theta^{t, 0}_{k, d} - \Theta^*_{k, d} )$ as:
\begin{align}
{\textbf{G}}^{t, \tau}_{k, d}  (\Theta^{t, 0}_{k, d} - \Theta^*_{k, d} ) \geq {\textbf{G}}^{t, 0}_{k, d} (\Theta^{t, 0}_{k, d} - \Theta^*_{k, d} ) - 2 \eta \beta \epsilon E L
\end{align}
Finally, we can get the lower bound of $( {\textbf{G}}^{t, \tau})^{\top} (\Theta^{t, 0} - \Theta^*)$:
\begin{align}
& ( {\textbf{G}}^{t, \tau})^{\top} (\Theta^{t, 0} - \Theta^*)  = \sum_{d=1}^D {\textbf{G}}^{t, \tau}_{k, d} (\Theta^{t, 0}_{k, d} - \Theta^*_{k, d}) \\
& \geq \sum_{d=1}^D {\textbf{G}}^{t, 0}_{k, d} (\Theta^{t, 0}_{k, d} - \Theta^*_{k, d} ) - 2 \eta \beta D \epsilon E L \\
& \geq ({\textbf{G}}^{t, 0})^{\top} (\Theta^{t, 0} - \Theta^*) - 2 \eta \beta D \epsilon E L
\end{align}
The end of the proof of Lemma 1.

\section{Proof of Theorem 1}
\label{appex-theo1}
In this section, we will present the detailed proof procedure for Theorem 1, which addresses the scenario of OVFL without the application of quantization methods. According to the overall model update rule of OVFL, we can get
$\Theta^{t+ 1, 0} = \Theta^{t, 0} - \eta  \sum_{\tau = 0}^{E -1} {\textbf{G}}^{t, \tau}$.
Then we obtain the subsequent inequality:
\begin{align}
&  \left [ \left \| \Theta^{t+1, 0} - \Theta^* \right \|^2 \right ] \\
= &  \left \| \Theta^{t, 0} - \Theta^* \right \|^2 + \eta^2 \left [ \left \|   \sum_{\tau = 0}^{E -1} {\textbf{G}}^{t, \tau} \right \|^2 \right ] \notag \\
- & 2 \eta  \left [ \sum_{\tau = 0}^{E -1} (  {\textbf{G}}^{t, \tau})^{\top} (\Theta^{t, 0} - \Theta^*) \right ] \label{the1-ex}
\end{align}
Then we can further get the bound of $ \left \|   \sum_{\tau = 0}^{E -1} {\textbf{G}}^{t, \tau} \right \|^2 $ as:
\begin{align}
\left \|   \sum_{\tau = 0}^{E -1} {\textbf{G}}^{t, \tau} \right \|^2  \leq  E\sum_{\tau = 0}^{E -1} \left \|  {\textbf{G}}^{t, \tau} \right \|^2 \label{jse-th1}
\end{align}

Where Eq.~(\ref{jse-th1}) is due to triangle inequality. Plugging  Eq.~(\ref{jse-th1}) into Eq.~(\ref{the1-ex}), we can obtain the upper-bound as:
\begin{align}
&  \left [ \left \| \Theta^{t+1, 0} - \Theta^* \right \|^2 \right ] \\
& \leq  \left \| \Theta^{t, 0} - \Theta^* \right \|^2 - 2 \eta \left [ \sum_{\tau = 0}^{E -1} ( {\textbf{G}}^{t, \tau})^{\top} (\Theta^{t, 0} - \Theta^*) \right ] \notag\\
& + \eta^2 E \sum_{\tau = 0}^{E -1} \left \|  {\textbf{G}}^{t, \tau} \right \|^2
\end{align}
Then we can further obtain the following upper-bound:
\begin{align}
& \sum_{\tau = 0}^{E -1} ( {\textbf{G}}^{t, \tau})^{\top} (\Theta^{t, 0} - \Theta^*) \label{inter-eq}\\
\leq & \frac{1}{2 \eta } ( \left \| \Theta^{t, 0} - \Theta^* \right \|^2 -  \left [ \left \| \Theta^{t+1, 0} - \Theta^* \right \|^2 \right ]) \notag \\
+ &  \frac{\eta E}{2} \sum_{\tau = 0}^{E -1} \left \|  {\textbf{G}}^{t, \tau} \right \|^2
\end{align}
Subsequently, it is necessary to establish the lower bound of Eq.~(\ref{inter-eq}). 
Then with the convexity of loss function in Assumption 1, we can get the lower-bound as:
\begin{align}
 ({\textbf{G}}^{t, 0})^{\top} (\Theta^{t, 0} - \Theta^*) \geq \left [ F_t (\Theta^{t, 0}) -  F_t (\Theta^*)  \right ] \label{eq:convex}
\end{align}

Then with the help of Lemma 1, we can further get:
\begin{align}
& ( {\textbf{G}}^{t, \tau})^{\top} (\Theta^{t, 0} - \Theta^*) \notag \\
& \geq   ({\textbf{G}}^{t, 0})^{\top} (\Theta^{t, 0} - \Theta^*) - 2 \eta \beta D \epsilon E L \label{eq:lma1}
\end{align}

We derive the lower-bound of Eq.(\ref{inter-eq}) by integrating Eq.(\ref{eq:convex}) and Eq.~(\ref{eq:lma1}), as follows:
\begin{align}
& \sum_{\tau = 0}^{E -1}  ( {\textbf{G}}^{t, \tau})^{\top} (\Theta^{t, 0} - \Theta^*) \\
& \geq E   \left [ F_t (\Theta^{t, 0}) -  F_t (\Theta^*)  \right ] - 2 \eta \beta D \epsilon E^2 L
\end{align}
Utilizing $\sum_{\tau = 0}^{E -1}  ( {\textbf{G}}^{t, \tau})^{\top} (\Theta^{t, 0} - \Theta^*)$ as a transitional term and combining both the upper and lower bounds, then taking the expectation $\mathbb{E}_{t}$ respect to $t$ on both sides and summing over $t = 1, 2, \dots, T$, we derive the subsequent regret bound:
\begin{align}
& \sum_{t=1}^{T} \mathbb{E}_{t} \left [ F_t (\Theta^{t, 0})  \right ]  -  \sum_{t=1}^{T} F_t (\Theta^*) \\
\leq &  \frac{1}{2 \eta E } \sum_{t=1}^{T} ( \mathbb{E}_{t}  \| \Theta^{t, 0} - \Theta^* \|^2 -  \mathbb{E}_{t+1} \left [ \left \| \Theta^{t+1, 0} - \Theta^* \right \|^2 \right ]) \notag \\
+ &  \frac{\eta}{2 } \sum_{t=1}^{T} \sum_{\tau = 0}^{E -1} \mathbb{E}_{t}  \left \| \textbf{G}^{t, \tau} \right \|^2 + 2 T \eta \beta D \epsilon E L\\
\overset{(a)}{\leq} & \frac{1}{2 \eta E }   \left \| \Theta^{1, 0} - \Theta^* \right \|^2  +  2 T \eta \beta D \epsilon E L \notag \\
+  &  \frac{\eta}{2 } \sum_{t=1}^{T} \sum_{\tau = 0}^{E -1} \left \| \left [ (\nabla_0  F_t ({\Phi}^{t, \tau}_0))^{\top}, \dots, (\nabla_K  F_t ({\Phi}^{t, \tau}_K))^{\top} \right ]^{\top} \right \|^2 \\
\overset{(b)}{\leq} & \frac{  \left \| \Theta^{1, 0} - \Theta^* \right \|^2}{2 \eta E }  + \frac{\eta T E K L^2}{2} + 2 \eta T \beta D \epsilon E L
\end{align}
Where (a)  results from the summation eliminating the intermediate term, and (b) holds when Assumption \ref{assm:bpd} is applied. Then we finish the proof of Theorem 1.

\section{Proof of Lemma 2}
In the Lemma 2, our objective is to establish the validity of the following inequality.
\begin{align}
( \hat{\textbf{G}}^{t, \tau})^{\top} (\Theta^{t, 0} - \Theta^*) \geq   ({\textbf{G}}^{t, 0})^{\top} (\Theta^{t, 0} - \Theta^*) - 2  \beta D (\eta \epsilon E L + \rho) \notag
\end{align}
Next, we detail the specific steps of the proof. Firstly, according to Assumption \ref{assm:model-variant}, we can further obtain:
\begin{align}
\left | \Theta^{t, 0}_{k, d} - \Theta^*_{k, d} \right | \leq 2 \beta
\end{align}
By integrating the aforementioned inequality with Assumption \ref{assm:gradient-change}, we can deduce the following:
\begin{align}
& \left |(\hat{\textbf{G}}^{t, \tau}_{k, d}  - {\textbf{G}}^{t, 0}_{k, d} ) (\Theta^{t, 0}_{k, d} - \Theta^*_{k, d} )\right | \\
= & \left |(\hat{\textbf{G}}^{t, \tau}_{k, d} - {\textbf{G}}^{t, \tau}_{k, d} + {\textbf{G}}^{t, \tau}_{k, d} - {\textbf{G}}^{t, 0}_{k, d} ) (\Theta^{t, 0}_{k, d} - \Theta^*_{k, d} )\right | \\
\leq & \left |(\hat{\textbf{G}}^{t, \tau}_{k, d} - {\textbf{G}}^{t, \tau}_{k, d} ) (\Theta^{t, 0}_{k, d} - \Theta^*_{k, d} )\right | \notag  \\
+ & \left | ({\textbf{G}}^{t, \tau}_{k, d} - {\textbf{G}}^{t, 0}_{k, d} ) (\Theta^{t, 0}_{k, d} - \Theta^*_{k, d} )\right | \\
\leq & 2 \eta \beta \epsilon E L + 2  \beta \rho \label{eq:all_error}
\end{align}
The derivation of Eq.~(\ref{eq:all_error}) is partly based on the previous results presented in Lemma 1 and partly on the stipulations of Assumption 5. Then we can get the lower bound as:
\begin{align}
&  \hat{\textbf{G}}^{t, \tau}_{k, d}  (\Theta^{t, 0}_{k, d} - \Theta^*_{k, d} ) \notag \\
& \geq {\textbf{G}}^{t, 0}_{k, d} (\Theta^{t, 0}_{k, d} - \Theta^*_{k, d} ) - 2  \beta (\eta \epsilon E L + \rho)
\end{align}
Finally, we can get the lower bound of $( \hat{\textbf{G}}^{t, \tau})^{\top} (\Theta^{t, 0} - \Theta^*)$:
\begin{align}
& ( \hat{\textbf{G}}^{t, \tau})^{\top} (\Theta^{t, 0} - \Theta^*)  = \sum_{d=1}^D \hat{\textbf{G}}^{t, \tau}_{k, d} (\Theta^{t, 0}_{k, d} - \Theta^*_{k, d}) \\
& \geq \sum_{d=1}^D {\textbf{G}}^{t, 0}_{k, d} (\Theta^{t, 0}_{k, d} - \Theta^*_{k, d} ) - 2  \beta D (\eta \epsilon E L + \rho) \\
& \geq ({\textbf{G}}^{t, 0})^{\top} (\Theta^{t, 0} - \Theta^*) - 2  \beta D (\eta \epsilon E L + \rho)
\end{align}
Here we finish the proof of Lemma 2.

\section{Proof of Theorem 2}
\label{appex-theo2}
In this section, we will present the detailed proof procedure for Theorem 2, which addresses the scenario of OVFL with the application of quantization methods. According to the overall model update rule of OVFL, we can get
$\Theta^{t+ 1, 0} = \Theta^{t, 0} - \eta  \sum_{\tau = 0}^{E -1} \hat{\textbf{G}}^{t, \tau}$.
Then we obtain the subsequent inequality:
\begin{align}
&  \left [ \left \| \Theta^{t+1, 0} - \Theta^* \right \|^2 \right ] \\
= &  \left \| \Theta^{t, 0} - \Theta^* \right \|^2 + \eta^2 \left [ \left \|   \sum_{\tau = 0}^{E -1} \hat{\textbf{G}}^{t, \tau} \right \|^2 \right ] \notag \\
- & 2 \eta  \left [ \sum_{\tau = 0}^{E -1} (  \hat{\textbf{G}}^{t, \tau})^{\top} (\Theta^{t, 0} - \Theta^*) \right ] \label{eq13}
\end{align}
Then we can further get the bound of $ \left \|   \sum_{\tau = 0}^{E -1} \hat{\textbf{G}}^{t, \tau} \right \|^2 $ as:
\begin{align}
& \left \|   \sum_{\tau = 0}^{E -1} \hat{\textbf{G}}^{t, \tau} \right \|^2 \leq  E\sum_{\tau = 0}^{E -1} \left \|  \hat{\textbf{G}}^{t, \tau} \right \|^2 \label{jse} \\
& \leq 2 E \sum_{\tau = 0}^{E -1} \left \|  \hat{\textbf{G}}^{t, \tau} - {\textbf{G}}^{t, \tau} \right \|^2 + 2 E \sum_{\tau = 0}^{E -1} \left \|   {\textbf{G}}^{t, \tau} \right \|^2 \label{quanti_part1}
\end{align}
Where Eq.~(\ref{jse}) is due to triangle inequality, and Eq.~(\ref{quanti_part1}) is due to the face $\left \| A+B \right \|^2 \leq 2 \left \| A \right \|^2 + 2 \left \| B \right \|^2$. Plugging  Eq.~(\ref{quanti_part1}) into Eq.~(\ref{eq13}), we can obtain the upper-bound as:
\begin{align}
&  \left [ \left \| \Theta^{t+1, 0} - \Theta^* \right \|^2 \right ] \\
& \leq  \left \| \Theta^{t, 0} - \Theta^* \right \|^2 - 2 \eta \left [ \sum_{\tau = 0}^{E -1} ( \hat{\textbf{G}}^{t, \tau})^{\top} (\Theta^{t, 0} - \Theta^*) \right ] \notag\\
& + 2 \eta^2 E \sum_{\tau = 0}^{E -1} \left \|  \hat{\textbf{G}}^{t, \tau} - {\textbf{G}}^{t, \tau} \right \|^2 + 2 \eta^2 E \sum_{\tau = 0}^{E -1} \left \|   {\textbf{G}}^{t, \tau} \right \|^2 \label{double_expand}
\end{align}
Then we can further obtain the following upper-bound:
\begin{align}
& \sum_{\tau = 0}^{E -1} ( \hat{\textbf{G}}^{t, \tau})^{\top} (\Theta^{t, 0} - \Theta^*) \label{inter-eql2}\\
\leq & \frac{1}{2 \eta } ( \left \| \Theta^{t, 0} - \Theta^* \right \|^2 -  \left [ \left \| \Theta^{t+1, 0} - \Theta^* \right \|^2 \right ]) \notag \\
+ &   \eta E \sum_{\tau = 0}^{E -1} \left \|  \hat{\textbf{G}}^{t, \tau} - {\textbf{G}}^{t, \tau} \right \|^2 +  \eta E \sum_{\tau = 0}^{E -1} \left \|   {\textbf{G}}^{t, \tau} \right \|^2
\end{align}
Subsequently, it is necessary to establish the lower bound of Eq.~(\ref{inter-eql2}). 
Then with the Assumption 1, we can get the lower-bound with the convexity of the loss function as:
\begin{align}
 ({\textbf{G}}^{t, 0})^{\top} (\Theta^{t, 0} - \Theta^*) \geq \left [ F_t (\Theta^{t, 0}) -  F_t (\Theta^*)  \right ] 
\end{align}

Then based on Lemma 2, we can further get:
\begin{align}
& ( \hat{\textbf{G}}^{t, \tau})^{\top} (\Theta^{t, 0} - \Theta^*)  \notag\\
& \geq   ({\textbf{G}}^{t, 0})^{\top} (\Theta^{t, 0} - \Theta^*) - 2  \beta D (\eta \epsilon E L + \rho)
\end{align}

Then we can get the lower-bound of Eq.~(\ref{inter-eql2}) as:
\begin{align}
& \sum_{\tau = 0}^{E -1}  ( \hat{\textbf{G}}^{t, \tau})^{\top} (\Theta^{t, 0} - \Theta^*) \\
& \geq E   \left [ F_t (\Theta^{t, 0}) -  F_t (\Theta^*)  \right ] - 2  E \beta D (\eta \epsilon E L + \rho)
\end{align}
Utilizing $\sum_{\tau = 0}^{E -1}  ( \hat{\textbf{G}}^{t, \tau})^{\top} (\Theta^{t, 0} - \Theta^*)$ as a transitional term and combining both the upper and lower bounds, then taking the expectation $\mathbb{E}_{t}$ respect to $t$ on both sides and summing over $t = 1, 2, \dots, T$, we derive the subsequent regret bound:
\begin{align}
& \sum_{t=1}^{T} \mathbb{E}_{t} \left [ F_t (\Theta^{t, 0})  \right ]  -  \sum_{t=1}^{T} F_t (\Theta^*) \\
\leq &  \frac{1}{2 \eta E } \sum_{t=1}^{T} ( \mathbb{E}_{t}  \| \Theta^{t, 0} - \Theta^* \|^2 -  \mathbb{E}_{t+1} \left [ \left \| \Theta^{t+1, 0} - \Theta^* \right \|^2 \right ]) \notag \\
+ &   \eta \sum_{t=1}^{T} \sum_{\tau = 0}^{E -1} \left \|  \hat{\textbf{G}}^{t, \tau} - {\textbf{G}}^{t, \tau} \right \|^2 +  \eta \sum_{t=1}^{T} \sum_{\tau = 0}^{E -1} \left \|   {\textbf{G}}^{t, \tau} \right \|^2 \notag \\
+ &  2  T \beta D \eta \epsilon E L + 2  T \beta D \rho\\
\overset{(c)}{\leq} & \frac{1}{2 \eta E }   \left \| \Theta^{1, 0} - \Theta^* \right \|^2  + \eta T E D \rho^2 \notag \\
+  &  \eta \sum_{t=1}^{T} \sum_{\tau = 0}^{E -1} \left \| \left [ (\nabla_0  F_t ({\Phi}^{t, \tau}_0))^{\top}, \dots, (\nabla_K  F_t ({\Phi}^{t, \tau}_K))^{\top} \right ]^{\top} \right \|^2 \notag \\
+  & 2  T \beta D \eta \epsilon E L + 2  T \beta D \rho\\
\overset{(d)}{\leq} & \frac{  \left \| \Theta^{1, 0} - \Theta^* \right \|^2}{2 \eta E }  + \eta T E K L^2  + \eta T E D \rho^2 \notag\\
+  &  2  \eta T \beta D  \epsilon E L + 2  T \beta D \rho
\end{align}
Where (c) results from the summation eliminating the intermediate term and the application of Assumption \ref{assm:gradient-quanti}. (d) holds when Assumption \ref{assm:bpd} is applied. Then we finish the proof of Theorem 2.
\end{document}